\documentclass[twoside]{article}

\usepackage{graphicx}
\usepackage{subcaption}
\usepackage{amsmath}
\usepackage{amssymb}
\usepackage{amsthm}
\usepackage[]{hyperref}
\usepackage{booktabs}
\usepackage{todonotes}
\usepackage{thm-restate}
\usepackage[numbers]{natbib}
\usepackage{mathtools}
\usepackage{algorithm}
\usepackage{algpseudocode}
\usepackage{algorithmicx}
\usepackage{float}
\usepackage[noabbrev]{cleveref}
\usepackage{multirow}
\usepackage{balance, xcolor}
\usepackage{microtype}
\usepackage{enumitem}

\newtheorem{lemma}{Lemma}
\newtheorem{example}{Example}
\newtheorem{definition}{Definition}

\DeclareMathOperator*{\argmin}{arg\,min}
\DeclareMathOperator*{\sgn}{\mathrm{sign}}

\newlength{\defbaselineskip}
\newcommand{\mailto}[1]{\href{mailto:#1@stanford.edu}{#1}}

\title{Minimizing Close-k Aggregate Loss Improves Classification}
\author{
  Bryan He,
  James Zou\\
  Stanford University \\
  \texttt{\{\mailto{bryanhe},\mailto{jamesz}\}@stanford.edu}
}
\date{}

\begin{document}

\maketitle

\begin{abstract}
In classification, the de facto method for aggregating individual losses  is the average loss.
When the actual metric of interest is 0-1 loss, it is common to minimize the average surrogate loss for some well-behaved (e.g. convex) surrogate.
Recently, several other aggregate losses such as the maximal loss and average top-$k$ loss were proposed as alternative objectives to address shortcomings of the average loss.
However, we identify common classification settings---e.g. the data is imbalanced, has too many easy or ambiguous examples, etc.---when average, maximal and average top-$k$ all suffer from suboptimal decision boundaries, even on an infinitely large training set.
To address this problem, we propose a new classification objective called the close-$k$ aggregate loss, where we adaptively minimize the loss for points close to the decision boundary.
We provide theoretical guarantees for the 0-1 accuracy when we optimize close-$k$ aggregate loss.
We also conduct systematic experiments across the PMLB and OpenML benchmark datasets. Close-$k$ achieves significant gains in 0-1 test accuracy---improvements of $\geq 2$\% and $p<0.05$---in over 25\% of the datasets compared to average, maximal and average top-$k$.
In contrast, the previous aggregate losses outperformed close-$k$ in less than 2\% of the datasets. 
\end{abstract}

\section{Introduction}

In the supervised learning setting, we aim to learn a function $f: \mathcal{X} \rightarrow \mathcal{Y}$
based on a labeled training set $S = \{(\mathbf{x}_i, y_i)\}_{i=1}^{n}$.
Commonly, $f$ is parameterized by a vector $\theta$.

We learn $f$ by minimizing an aggregate loss over the individual losses of the elements in the training set, along with a regularization term to penalize the complexity of $f$.
This can be formulated as $\argmin_\theta L(\{\ell(y_i, f_\theta(x_i))\mid i\in[n]\}) + \Lambda(\theta)$, where $\ell(y, \hat{y})$ is the individual loss of a prediction $\hat{y}$ when the true label is $y$, $L$ aggregates the individual losses, and $\Lambda$ is a measure of the complexity of $f_\theta$.

A standard setting in supervised learning is classification, where $\mathcal{Y}$ is a discrete set.
In classification problems, we are commonly interested in the 0-1 loss, ${\mathbf{1}[y \neq \hat{y}]}$.
However, the 0-1 loss is difficult to optimize, so surrogate losses such as the logistic loss, $\log(1 + \exp(-y\hat{y}))$, and the hinge loss, $\max(0, 1-y\hat{y})$, and are commonly used instead.

\begin{figure*}[htbp]
    \begin{subfigure}[b]{0.246\textwidth}
        \begin{center}
            \includegraphics[scale=1.0]{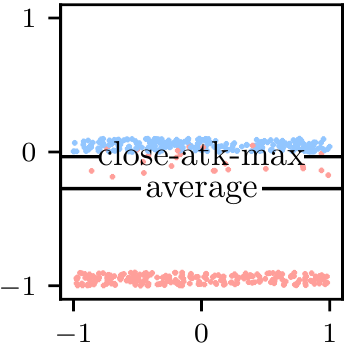}
            \caption{Easy examples}
            \label{fig: easy}
        \end{center}
    \end{subfigure}
    \begin{subfigure}[b]{0.246\textwidth}
        \begin{center}
            \includegraphics[scale=1.0]{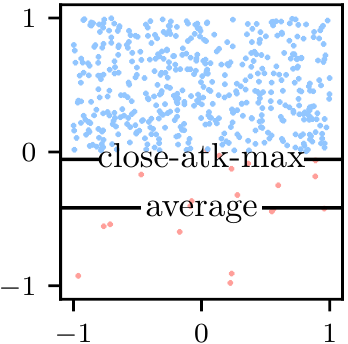}
            \caption{Imbalance}
            \label{fig: imbalance}
        \end{center}
    \end{subfigure}
    \begin{subfigure}[b]{0.246\textwidth}
        \begin{center}
            \includegraphics[scale=1.0]{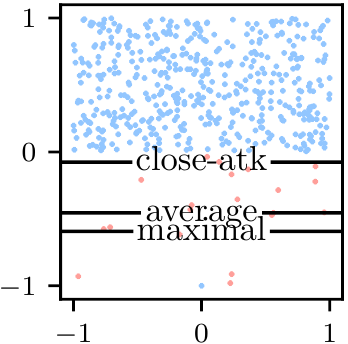}
            \caption{Imbalance + outlier}
            \label{fig: outlier}
        \end{center}
    \end{subfigure}
    \begin{subfigure}[b]{0.246\textwidth}
        \begin{center}
            \includegraphics[scale=1.0]{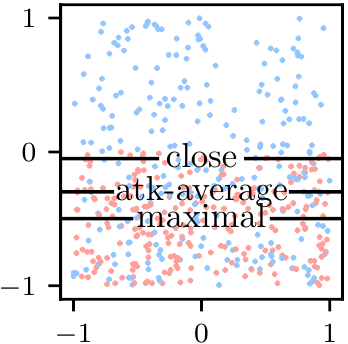}
            \caption{Ambiguous examples}
            \label{fig: ambiguous}
        \end{center}
    \end{subfigure}
    \caption{Comparison of various aggregate losses on synthetic datasets using logistic loss as the individual loss. In all of these datasets, the optimal 0-1 error is achieved by splitting the vertical axis at 0.}
    \label{fig: synth}
\end{figure*}

The most common aggregate loss is the average loss $L(\{\ell_i\}_{i=1}^{n}) = \frac{1}{n}\sum_{i=1}^{n} \ell_i$, which corresponds to the empirical risk minimization setting \cite{vapnik1992nips}.
Despite its ubiquitous use, aggregating individual logistic losses with an average can result in suboptimal classification even on the training set.
For example, when many easy examples are present (Figure~\ref{fig: easy}) or when the data set is imbalanced (Figure~\ref{fig: imbalance}), average loss result in errors even on the training set.
In both of these cases, the average loss is rewarded for continuing to decrease the error on correctly classified examples, even at the expense of introducing errors on the other class.

Since the use of average loss is so common, many techniques, including robust estimation \cite{huber1964aos}, and more recently, the maximal and average top-$k$ losses \cite{shalevshwartz2016icml,fan2017nips}, have been studied as alternatives to address its shortcomings.
Interestingly, these examples highlight issues that are unaddressed by robust estimation techniques, which focus on cases where suboptimal classification results from the points with large errors.

The maximal loss is able to handle these cases, but it is very sensitive to outliers (Figure~\ref{fig: outlier}) \cite{shalevshwartz2016icml}.
The average top-$k$ loss was proposed to allow a tradeoff between the average and maximal losses \cite{fan2017nips}.
However, in the presence of ambiguous examples that cannot be classified correctly, the average top-$k$ loss is unable to make an optimal classification (Figure~\ref{fig: ambiguous}).
This is caused by the fact that the average top-$k$ loss is still rewarded for reducing the loss of the ambiguous examples, even though it is not possible to classify those examples correctly.

We now note that the average logistic and the average hinge loss are both classification calibrated losses~\cite{lin2004spl,bartlett2006jasa}.
A key property of classification calibrated losses is that its minimizer leads to the Bayes optimal classification rule---that is, its minimizer achieves the Bayes risk, which is the minimum possible error.
At a glance, this may appear to imply that the average logistic loss (corresponding to logistic regression) and the average hinge loss (corresponding to linear SVMs) will obtain the optimal accuracy for a linear decision boundary as data size approaches infinity.
However, it is clear that this is not the case from the examples in Figure~\ref{fig: synth}.
This is due to the fact that classification calibration is concerned with the setting where the classification function is selected from set of all measurable functions.
This is not a typical setting---in practice, supervised learning selects from a restricted set of functions: for example, for logistic regression and linear SVMs, the decision boundary must be a linear function.
Due to this difference, both logistic regression and linear SVMs can result in suboptimal classifications, even in the simple cases in Figure~\ref{fig: synth}.


\paragraph{Our contributions} In Section \ref{sec: motivation}, we provide additional background for classification calibration and explain why it is not sufficient to capture the suboptimal behavior in Figure \ref{fig: synth}. We introduce a more precise definition that is capable of describing the behavior in these cases. In Section \ref{sec: method}, we introduce the close-$k$ aggregate loss, which focuses on the examples near the decision boundary, which is capable of handling these cases. In Section \ref{sec: exp}, we provide extensive experimental evidence, and find that we obtain statistically significant improvements on more than 20\% of standard benchmark datasets.
We provide theoretical guarantees and generalization bounds for the close-$k$ loss in Section \ref{sec: analysis}.

\section{Motivation}
\label{sec: motivation}

In typical binary classification problems, we are interested in the 0-1 loss.
In this case, the \textit{Bayes optimal risk} is the minimum possible loss, which is achieved by the \textit{Bayes optimal decision rule} \cite{lin2004spl,bartlett2006jasa}.
A loss is \textit{classification calibrated} if its minimizer has the same sign as the Bayes decision rule, which implies that it achieves the Bayes risk.
We note that the average loss and average \mbox{top-$k$} loss, for sufficiently large $k$, are both classification calibrated~\cite{fan2017nips}.

More formally, the Bayes optimal decision rule is $f^*(x) = \sgn\left(2\eta(x) - 1\right)$, where \mbox{$\eta(x) = \Pr(Y = 1\mid X = \mathbf{x})$} is the conditional probability of the positive class.
The Bayes optimal risk is then given by $R^* = \mathbb{E}\left[L_{0-1}(f^*)\right]$, where the expectation is over the data distribution.
Next, let $f_n$ be the minimizer of the aggregate loss over $S_n = \{(\mathbf{x}_i, y_i)\}_{i=1}^{n}$.
The loss is classification calibrated if $\lim_{n\rightarrow\infty} \mathbb{E}\left[L_{0-1}(f_n)\right] = R^*$.

However, as we previously noted, this does not correspond to a typical learning setting.
In general, we do not select our classification function from the set of all measureable functions.
Instead, we select a function from a parameterized family $\Omega$.
For example, the commonly used logistic regression and linear SVM both restrict the classification function to be a linear function.
To better understand this setting, we introduce a more precise definition for classification calibration:
\begin{definition}
    An aggregate loss function is classification calibrated with respect to a family of classification functions $\Omega$ if and only if $\lim_{n\rightarrow\infty} \mathbb{E}\left[L_{0-1}(f_n)\right] = R^*$, where $f_n$ is the minimizer of the aggregate loss over $S_n$ in $\Omega$, and $R^* = \min_{f\in\Omega} \mathbb{E}\left[L_{0-1}(f)\right]$ is the optimal risk of a function in $\Omega$.
\end{definition}

Under this more precise definition, which more closely describes typical learning settings, we find that classification calibrated losses are not necessarily classification calibrated with respect to linear decision boundaries.
In the remainder of this section, we discuss two simple examples where average top-$k$, average, and maximal losses are not calibrated with respect to linear decision boundaries. 
Due to the fact that the average top-$k$ losses, which averages the largest $k$ losses, is a generalization of the average loss ($k=n$) and the maximal loss ($k=1$), we focus on it in the rest of this section.

\begin{example}
    Consider the training set consisting of $n$ instances of $(x=-1,y=-1)$, $n$ instances of $(x=+1, y=+1)$, 1 instance of $(x=+M, y=-1)$, and 1 instance of $(x=-M, y=+1)$, where $M\gg n$.
\end{example}

This training set consists of two outliers, but is otherwise fully separable by a linear decision boundary.
Suppose we use a linear decision boundary of the form $f(x) = wx+b$.
Due to the two outliers, the decision boundary is flipped when minimizing the average top-$k$ loss, for any choice of $k$, resulting in only the two outliers being classified correctly. 
As a result, average top-$k$ loss will result in a 0-1 loss of $2n$, rather than the optimal loss of 2.

\begin{example}
    Consider the training set consisting of $n$ negative ($y=-1$) instances with $x$ uniform between $-1$ and $1$, and $n$ positive ($y=+1$) instances with $x$ uniform between $0$ and $1$.
\end{example}

This example is closely related to the case in Figure~\ref{fig: synth}(d).
In this case, the optimal linear classifier gets an accuracy of $0.75$, which corresponds to a 0-1 loss of $n/2$, by classifying all points with $x<0$ as positive and all points with $x>0$ as negative.
However, optimizing this function for average top-$k$ loss will result in a 0-1 loss strictly greater than $n/2$.

The datasets in Figure \ref{fig: synth} and this section highlight several difficult cases that can cause suboptimal behavior:
\begin{itemize}[nolistsep]
    \item[$-$] imbalanced classes
    \item[$-$] easy examples far from the decision boundary
    \item[$-$] outliers
    \item[$-$] ambiguous examples that can't be classified correctly
\end{itemize}
We observe that these cases result in suboptimal behavior due to the fact that average loss is encouraged to decrease the individual loss in examples far from the decision boundary, even though the resulting accuracy of the classifier may decrease in response. 
Using this observation, we introduce the close-$k$ loss, which focuses on points close to the decision boundary, as a method to avoid these behaviors.

\section{Method}
\label{sec: method}

In this section, we introduce the close-$k$ aggregate loss and provide key properties of the loss.
In particular, we show that the close-$k$ loss provides a tradeoff between classification calibration and convexity.
Using this property, we provide a training scheme that is not sensitive to initialization.
We then present experimental results in \Cref{sec: exp}, and provide more detailed properties of the close-$k$ loss and generalization guarantees in \Cref{sec: analysis}.

\subsection{Close-k Aggregate Loss}
We propose to use the $k$ points closest to the decision boundary as an aggregate loss, which we call the \mbox{close-$k$} aggregate loss.
More formally, let $T$ be the threshold in the individual loss where an example is classified correctly.
For example, in logistic loss, $T=\log 2$, and in hinge loss, $T=1$.
Next, let $\ell_{[i]}$ denote the individual loss $\ell_j$ with the $i$-th smallest $|\ell_j - T|$---that is, the individual loss of the $i$-th closest example to the decision boundary.
The close-$k$ is defined as:
\begin{align*}
    &\phantom{{}={}}
    L_{\textrm{close-$k$}}(\{\ell_i\}_{i=1}^{n}) \\
    &=
    \sum_{i=1}^{n}
    \begin{dcases}
        \ell_{[i]} & \textrm{if $i \leq k$} \\
        0 & \textrm{if $i > k$ and $\ell_{[i]}$ correctly classified} \\
        M & \textrm{if $i > k$ and $\ell_{[i]}$ incorrectly classified}
    \end{dcases}
\end{align*}
where $M$ is a large constant.
We note that the choice of $M$ does not affect training, due that the fact that the it does not contribute to the gradient, but is required for deriving properties of the close-$k$ loss.
For these derivations, it is sufficient that $M$ is at least as large as the largest individual loss that occurs during training.

The close-$k$ aggregate loss can be trained via gradient descent.
The gradient of the close-$k$ loss is:
\begin{align*}
    &\phantom{{}={}}
    \frac{\partial}{\partial \theta} L_{\textrm{close-$k$}}(\{\ell_i\}_{i=1}^{n}) \\
    &=
    \sum_{i=1}^{n}
    \begin{dcases}
        \frac{\partial}{\partial \theta}\ell_{[i]} & \textrm{if $i\leq k$} \\
        0 & \textrm{otherwise}
    \end{dcases}
\end{align*}
Notice that gradient descent on this objective sets the gradients of points not near the boundary to zero, which allows this loss to be robust against the difficult cases in \Cref{sec: motivation}.

We now present several key properties of the close-$k$ loss.
For exposition, we only provide sketches of the proofs here, and defer full proofs to the analysis in \Cref{sec: analysis} and the Appendix.
The first property we are interested in is classification calibration. 
We find that for $k=1$, the close-$1$ loss is classification calibrated under restriction to any set of functions:
\begin{restatable}{theorem}{calib}
    \label{thm: calib}
    $L_{\textrm{close-$1$}}$ is classification calibrated when restricted to any $\Omega$ for sufficiently large $M$.
    That is, if $f_n$ is the minimizer of $L_{\textrm{close-$1$}}$ on $S_n$ in $\Omega$, then $\lim_{n\rightarrow \infty}\mathrm{E}[L_{0-1}(f_n)] = \min_{f\in\Omega}\mathrm{E}[L_{0-1}(f)]$. 
\end{restatable}
At a high level, this result holds due to the fact that the close-$1$ loss is tightly bounded by the 0-1 loss.
Thus, the minimizer of the close-$1$ loss is also the minimizer of the 0-1 loss.
We discuss the proof in more detail in \Cref{sec: analysis}.

A concern for the usage of close-$1$ loss is the non-convex nature of the loss.
We now show that larger choices of $k$ can be used as a tradeoff between calibration and convexity.
First, while the close-$k$ aggregate loss is not necessarily classification calibrated when restricted, its suboptimality can be tightly bounded.
\begin{restatable}{lemma}{lemmak}
    \label{lemma: k}
    Let
    \begin{align*}
        \hat\theta = \argmin_{\theta} L_{\textrm{close-$k$}}(\{(y_i, f_\theta(x_i))\}_{i=1}^{n}).
    \end{align*}
    Then,
    \begin{align*}
        L_{0-1}(\hat\theta)
        \leq
        \min_\theta L_{0-1}(\{(y_i, f_\theta(x_i))\}_{i=1}^{n}) + k - 1
    \end{align*}
\end{restatable}
The proof is similar to the proof of \Cref{thm: calib} and is provided in the Appendix.
Additionally, for $k=n$, the close-$n$ loss is exactly the average loss, which is convex as long as the individual losses are convex.

Based on these observations, we see that the selection of $k$ is a tradeoff between being calibrated and being convex.
To take advantage of this observation, we can begin training with large values of $k$, where the loss is convex, followed by decaying $k$ to the desired value.
We describe this learning scheme as \Cref{alg: decay}.
In close decay, in the first third of the epochs, we optimize the model parameters using the standard average loss (corresponding to $k=n$).
For the middle third, we decrease $k$ linearly from $n$ to a base value $k^*$.
Finally we train the last third of epochs for $k^*$ 
Due to the fact that the initial choice of $k=n$ results in a convex loss, this also makes \Cref{alg: decay} robust to the initial parameter choice.
In practice, we treat $k^*$ as a hyperparameter that we search for in powers of 10, i.e. $k^* \in \{10, 10^2, ..., n\}$.
For each gradient step, the examples with losses closest to the threshold are selected, and the gradients for the corresponding individual losses are computed and added.

\begin{algorithm}[htbp]
    \caption{Close-$k$ Aggregate Loss with Decaying $k$}
    \label{alg: decay}
    \begin{algorithmic}[1]
        \State $\theta\gets \textrm{Random Initialization}$
        \For{$i$ from 1 to epochs}
            \If {$i < \textrm{epochs} / 3$}
                \State $k = n$
            \ElsIf {$i < \frac{2}{3}\textrm{epochs}$}
                \State $k = k^* + \textrm{round}((n - k^*)\frac{2\textrm{epochs} - 3i}{\textrm{epochs}}$
            \Else:
                \State $k = k^*$
            \EndIf

            \State Run gradient descent step with $k$
        \EndFor
        \State \textbf{return} $\theta$
    \end{algorithmic}
\end{algorithm}

\begin{table*}[htbp]
    \small
    \caption{Fraction of PMLB datasets with significant improvement ($p\leq0.05$). The entry in row $i$ and column $j$ of the table indicates the fraction of datasets where method $j$ significantly outperformed method $i$. For example, close decay outperforms average top-$k$ (atk) in 29\% of the data in linear logistic regression. Increases in accuracy by close decay compared to existing methods are shown in bold, and decreases in accuracy are shown in italics.}
    \label{table: pmlb}
    \begin{center}
\begin{tabular}{llcccccccccc}
\toprule
 &  & \multicolumn{5}{c}{Logistic} & \multicolumn{5}{c}{Hinge} \\
\cmidrule(lr){3-7}\cmidrule(lr){8-12} 
& & close & close decay & atk & average & top & close & close decay & atk & average & top \\
\midrule
\multirow{5}{*}{Linear}
& close       &      & 0.04 & 0.00 & 0.00 & 0.03 &      & 0.08 & 0.00 & 0.00 & 0.01 \\
& close decay & 0.00 &      & \textit{0.00} & \textit{0.00} & \textit{0.02} & 0.00 &      & \textit{0.00} & \textit{0.00} & \textit{0.01} \\
& atk         & 0.26 & \textbf{0.29} &      & 0.01 & 0.06 & 0.20 & \textbf{0.26} &      & 0.05 & 0.03 \\
& average     & 0.27 & \textbf{0.27} & 0.03 &      & 0.08 & 0.22 & \textbf{0.25} & 0.04 &      & 0.08 \\
& top         & 0.28 & \textbf{0.32} & 0.10 & 0.15 &      & 0.27 & \textbf{0.28} & 0.10 & 0.12 &      \\
\midrule
\multirow{5}{*}{NN}
& close       &      & 0.06 & 0.00 & 0.00 & 0.01 &      & 0.09 & 0.01 & 0.03 & 0.04 \\
& close decay & 0.04 &      & \textit{0.02} & \textit{0.02} & \textit{0.03} & 0.02 &      & \textit{0.01} & \textit{0.00} & \textit{0.02} \\
& atk         & 0.04 & \textbf{0.06} &      & 0.02 & 0.03 & 0.02 & \textbf{0.03} &      & 0.00 & 0.01 \\
& average     & 0.04 & \textbf{0.08} & 0.02 &      & 0.03 & 0.06 & \textbf{0.09} & 0.03 &      & 0.06 \\
& top         & 0.26 & \textbf{0.28} & 0.23 & 0.23 &      & 0.29 & \textbf{0.30} & 0.24 & 0.27 &      \\
\bottomrule
\end{tabular}
\end{center}

\end{table*}

\section{Experiments}
\label{sec: exp}

\subsection{Setup}
We experimentally verify that our method has consistently strong performance on a wide range of datasets.
We run our experiments on all binary classification tasks in the PMLB suite~\cite{olson2017biomed} (94 datasets)
and the OpenML benchmark~\cite{bischl2017arxiv} (54 datasets).
Many datasets from the popular KEEL \cite{alcala2011jmlsc} and UCI ML \cite{asuncion2007uci} repositories are included in the PMLB suite.
These datasets have a wide range of examples and dimensions (see Figure \ref{fig: summary}).
We additionally report results on the datasets used by \citet{fan2017nips}, which proposed the average top-$k$ loss.
In our experiments, we find statistically significant ($p<0.05$) improvement in more than 20\% of the datasets when using linear classifiers.

\begin{figure}[htbp]
    \begin{center}
        \includegraphics[]{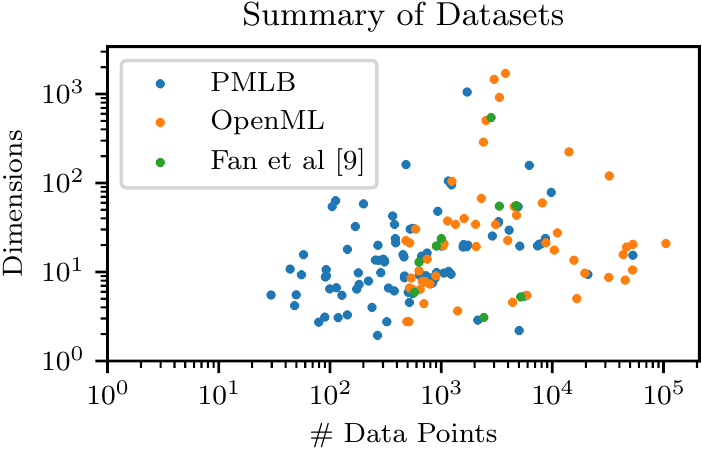}
        \caption{Number of data points and features for the datasets.}
        \label{fig: summary}
    \end{center}
\end{figure}

We run experiments using two different models.
The first model is a linear classifier.
In this case,  aggregating individual logistic losses with averaging corresponds to logistic regression, and aggregating individual hinge losses with averaging corresponds to a linear SVM.
The second model we use is a neural network with two hidden layers.
Each hidden layer has as many features as the original feature space.
We add a residual connection to allow direct access to the original features.

For each dataset, we randomly sample 25 different splits of the data into training, validation, and test sets.
For each split, we use 50\% of examples as the training set, 25\% of examples as the validation set, and 25\% of examples as the test set.
On each split, use the validation set to select hyperparameters.
The regularization factor $\Lambda$ is selected from $\{10^{-5},10^{-4},\ldots,10^{5}\}$.
The number of examples $k$ used by the aggregate loss is selected from $\{10,\ldots,10^{\lfloor\log_{10}n\rfloor},n\}$.
Even the largest dataset only requires 6 different values of $k$ to be searched.
The hyperparameter resulting in the highest accuracy on the validation set is selected.
Note that for both the average top-$k$ and top-$k$ losses\footnote{We use the top-$k$ loss (the $k$-th largest individual loss) as a slight generalization of maximal loss as suggested by \citet{fan2017nips}.}, the $k$ is also a hyperparameter which we select from the same set.

\subsection{Results}

We summarize the results of our experiments on PMLB in Table \ref{table: pmlb} and OpenML in Table \ref{table: openml}.
In these tables, for each pair of losses, we report the fraction of datasets in which one outperforms the other in terms of accuracy at a significance level of $p=0.05$. In the Appendix, we additionally report the fraction of data sets where one algorithm improves over another algorithm by at least 2\% in accuracy; the trend is very similar.
We find that for linear models, \Cref{alg: decay} outperforms average, average top-$k$ and top-$k$ on over 25\% of datasets, regardless of the individual loss used, and only is only outperformed on 2\% of datasets.
We note that on many datasets, the methods have similar performance---this is not surprising due to the fact that logistic regression and linear SVMs are strong baselines for finding linear decision boundaries.
We also find improved performance using neural networks, although the improvements in performance tend to be smaller than with linear decision boundaries. 

\begin{table*}[htbp]
    \small  
    \caption{Fraction of OpenML datasets with significant improvement ($p\leq 0.05$).}
    \label{table: openml}
    \begin{center}
\begin{tabular}{llcccccccccc}
\toprule
 &  & \multicolumn{5}{c}{Logistic} & \multicolumn{5}{c}{Hinge} \\
\cmidrule(lr){3-7}\cmidrule(lr){8-12} 
 && close & close decay & atk & average & top & close & close decay & atk & average & top \\
\midrule
\multirow{5}{*}{Linear}
& close       &      & 0.04 & 0.00 & 0.00 & 0.00 &      & 0.07 & 0.02 & 0.00 & 0.00 \\
& close decay & 0.02 &      & \textit{0.00} & \textit{0.00} & \textit{0.02} & 0.04 &      & \textit{0.00} & \textit{0.00} & \textit{0.00} \\
& atk         & 0.33 & \textbf{0.30} &      & 0.04 & 0.09 & 0.35 & \textbf{0.33} &      & 0.00 & 0.07 \\
& average     & 0.30 & \textbf{0.28} & 0.04 &      & 0.11 & 0.30 & \textbf{0.28} & 0.04 &      & 0.07 \\
& top         & 0.50 & \textbf{0.46} & 0.28 & 0.33 &      & 0.46 & \textbf{0.46} & 0.28 & 0.30 &      \\
\midrule
\multirow{5}{*}{NN}
& close       &      & 0.05 & 0.00 & 0.00 & 0.00 &      & 0.05 & 0.00 & 0.02 & 0.00 \\
& close decay & 0.00 &      & \textit{0.00} & \textit{0.00} & \textit{0.00} & 0.00 &      & \textit{0.00} & \textit{0.02} & \textit{0.00} \\
& atk         & 0.00 & \textbf{0.05} &      & 0.00 & 0.02 & 0.02 & \textbf{0.07} &      & 0.00 & 0.00 \\
& average     & 0.02 & \textbf{0.11} & 0.05 &      & 0.00 & 0.02 & \textbf{0.09} & 0.00 &      & 0.00 \\
& top         & 0.25 & \textbf{0.23} & 0.25 & 0.25 &      & 0.23 & \textbf{0.30} & 0.18 & 0.30 &      \\
\bottomrule
\end{tabular}
\end{center}

\end{table*}

Next, in Table \ref{table: acc}, we run experiments on the 8 datasets used by \citet{fan2017nips}, which proposed average top-$k$. These 8 datasets are presumably settings favorable to the average top-$k$ loss.
We find that even on these datasets, the close decay loss performed statistically equivalent to average top-$k$ on six datasets. On the remaining two, where there is a statistically significant difference in test accuracy, close decay had notable improvements in accuracy.
 For example, the monk dataset resulted in an improvement of more than 6\%.
\begin{table*}
    \small
    \caption{Error rate in datasets from \citet{fan2017nips}. Statistically significant differences are shown in bold.}
    \label{table: acc}
    \begin{center}
\begin{tabular}{lrrrrrrrrrr}
\toprule
 & \multicolumn{5}{c}{Logistic} & \multicolumn{5}{c}{Hinge} \\
\cmidrule(lr){2-6}\cmidrule(lr){7-11} 
 & close & close decay & atk & average & top & close & close decay & atk & average & top \\
\midrule
monk       & \textbf{11.37} & \textbf{11.01} & 18.45 & 19.86 & 18.73 & \textbf{11.36} & \textbf{11.44} & 17.05 & 17.78 & 18.56 \\
phoneme    & \textbf{21.90} & \textbf{21.67} & 23.00 & 23.29 & 23.11 & 21.51 & 21.56 & 22.06 & 23.01 & 23.21 \\
madelon    & 44.29 & 44.79 & 45.98 & 44.34 & 47.42 & 44.31 & 44.22 & 44.31 & 43.26 & 46.78 \\
spambase   &  7.30 &  7.33 &  7.51 &  7.62 &  8.40 &  7.59 &  7.41 &  7.22 &  7.38 &  8.33 \\
titanic    & 22.17 & 21.51 & 22.14 & 22.26 & 22.57 & 21.98 & 22.03 & 22.13 & 22.29 & 22.60 \\
australian & 12.86 & 13.22 & 12.69 & 12.65 & 12.89 & 13.09 & 13.36 & 13.21 & 12.65 & 13.05 \\
splice     & 15.90 & 16.37 & 16.24 & 15.83 & 16.22 & 16.41 & 16.42 & 16.16 & 15.89 & 16.12 \\
german     & 24.48 & 24.33 & 23.93 & 23.98 & 26.32 & 24.99 & 24.15 & 24.67 & 24.45 & 25.62 \\
\bottomrule
\end{tabular}
\end{center}

\end{table*}

\subsection{Analysis}

To help better understand our improvement, we also directly compare against average loss in \Cref{fig: perf}.
In both cases, we see that many datasets have similar performance using either of the losses---this is expected, due to the fact that average loss is a reasonable baseline method. For a substantial number of datasets, using the simple close decay algorithm achieves good improvements.

\begin{figure}[htbp]
    \begin{center}
        \begin{subfigure}[b]{0.49\columnwidth}
            \begin{center}
                 \includegraphics[scale=1.0]{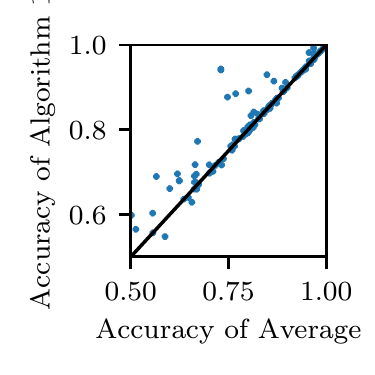}
                 \caption{Logistic}
            \end{center}
        \end{subfigure}
        \begin{subfigure}[b]{0.49\columnwidth}
            \begin{center}
                 \includegraphics[scale=1.0]{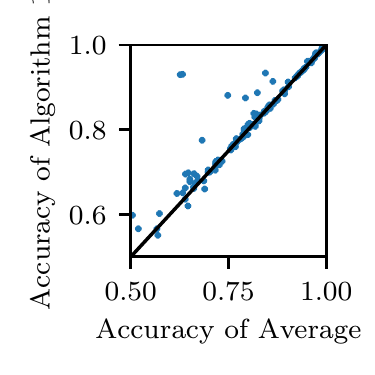}
                 \caption{Hinge}
            \end{center}
        \end{subfigure}
    \end{center}
    \caption{Comparison of performance of \Cref{alg: decay} and average loss using logistic and hinge loss. Each point is one dataset from PMLB and OpenML.}
    \label{fig: perf}
\end{figure}

Finally, we present a summary of the selected $k^*$ across datasets in \Cref{fig: k}.
To allow the value to be more easily interpreted, we show the selected value of $k$ normalized by the number of examples in the dataset.
In \Cref{fig: k}(a), we find that the selected $k^*$ covers a wide range across datasets, and in \Cref{fig: k}(a), we find that datasets where the selected $k^*$ is small are also the datasets where we see substantial improvements in test accuracy compared to average loss. This is  reasonable since if $k^*$ is close to $n$, then close decay is essentially using the average loss.

\begin{figure}[htbp]
    \begin{center}
        \begin{subfigure}[b]{0.49\columnwidth}
            \begin{center}
                \includegraphics[scale=1.0]{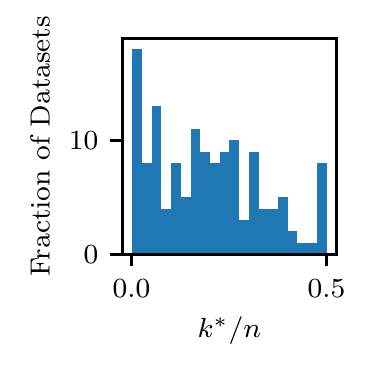}
                \caption{}
            \end{center}
        \end{subfigure}
        \begin{subfigure}[b]{0.49\columnwidth}
            \begin{center}
                \includegraphics[scale=1.0]{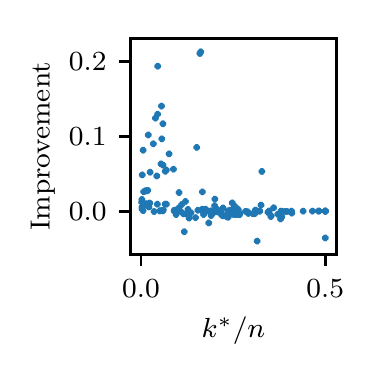}
                \caption{}
                \label{fig: improvement}
            \end{center}
        \end{subfigure}
        \caption{(a) Selected values of $k^*$. (b) Improvement on test accuracy vs. selected $k^*$. Improvement is measured compared to the accuracy of average loss.}
        \label{fig: k}
    \end{center}
\end{figure}

\subsection{Simulations}

\begin{figure*}[htbp]
    \begin{center}
        \begin{subfigure}[b]{0.26\textwidth}
            \begin{center}
                \includegraphics[]{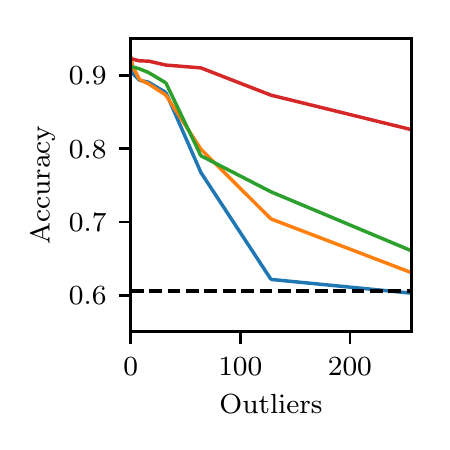}
            \end{center}
        \end{subfigure}
        \begin{subfigure}[b]{0.26\textwidth}
            \begin{center}
                \includegraphics[]{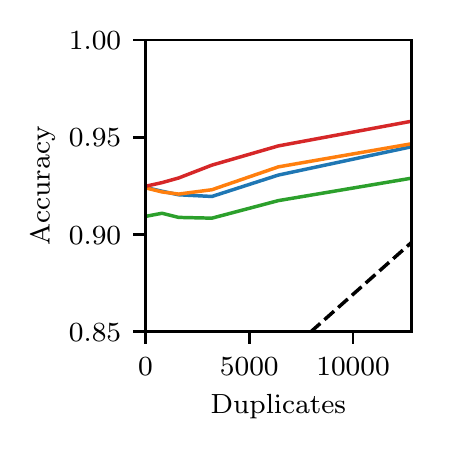}
            \end{center}
        \end{subfigure}
        \begin{subfigure}[b]{0.26\textwidth}
            \begin{center}
                \includegraphics[]{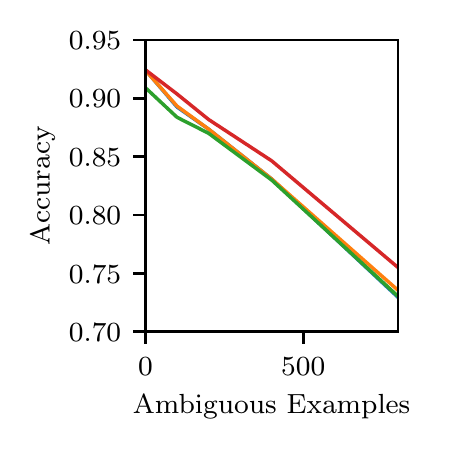}
            \end{center}
        \end{subfigure}
        \begin{subfigure}[b]{0.20\textwidth}
            \begin{center}
                \raisebox{1.5cm} {
                    \includegraphics[]{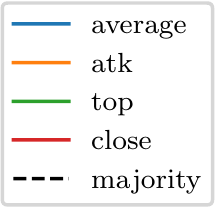}
                }
            \end{center}
        \end{subfigure}
        \caption{Simulated variants of the \texttt{spambase} dataset with  outliers, class imbalance, and ambiguous examples. The accuracy resulting from predicting the majority class on all examples is shown as a dotted line.}
        \label{fig: simulation}
    \end{center}
\end{figure*}

We now present several simulations in \Cref{fig: simulation} to further investigate the performance of close decay. 
These simulations allow us to avoid confounding factors, so that the source of the improvements are more clear.
In these simulations, we use the \texttt{spambase} dataset, which is the easiest example in \Cref{table: acc} (all methods average under 9\% error).
This allows us to increase the difficulty of the dataset to explore how the different methods respond to various challenges.
The spambase dataset consists of 4601 examples with 57 features, where 1813 of the examples are positive (39\% of the dataset).

We simulate three settings: outliers, class imbalance, and ambiguous examples.
To simulate outliers, we sample the class proportional to the original dataset.
We then sample one example with a matching label with features $x_1$ and one example with the opposite label with features $x_2$.
The outlier is then given $10x_2 - 9x_1$ as features (ie. the outlier is placed far from the decision boundary in the wrong direction).
To simulate class imbalance, we take the original dataset and randomly duplicate negative examples, which are already a slight majority class (61\%).
To simulate ambiguous examples, we randomly sample negative examples and create a copy with a positive label, resulting in a point that is impossible to always classify correctly.

First, in the case of outliers, we find that the close-$k$ loss is fairly robust, even when more than 5\% of the dataset is composed of outliers.
In contrast, all other methods have dropped more than 20 percentile points in accuracy, resulting in an accuracy that is only marginally better than predicting the majority class for all examples.

Next, in the case of class imbalance, we find that the close-$k$ loss results in improved performance compared to other aggregate losses as the class imbalance grows.
As the class imbalance grows, all methods improve in accuracy due to fact that predicting the majority class becomes increasingly accurate.

Finally, as the number of ambiguous examples increases, the close-$k$ loss maintains a higher accuracy than the other aggregate losses.
The results from this experiment closely resembles our example in \Cref{fig: synth}(d), where all methods other than the close-$k$ loss are encouraged to attempt to decrease the loss on the ambiguous examples despite the fact that it impossible to classify those examples correctly.

\section{Analysis}
\label{sec: analysis}

In this section, we provide theoretical guarantees for our method.
First, we show that the close-$k$ loss is closely bounded by the 0-1 loss, which provides intuition for some of the later results.
Next, we provide our proof of \Cref{thm: calib}.
Finally, we provide generalization guarantees for the close-$k$ loss.

\paragraph{Bounded by 0-1 Loss}
First, we show that the close-$k$ loss is closely bounded by the 0-1 loss.
Because of this property, even though the close-$k$ loss is not classification calibrated under restriction for $k > 1$, the close-$k$ loss still gives strong performance.
\begin{restatable}{lemma}{bound}\label{lem:bound}
    \label{lemma: bound}
    For all $S$ and $f$,
    \begin{align*}
        &\phantom{{}={}} L_{0-1}(\{(y_i, f_\theta(x_i))\}_{i=1}^{n}) - k
        \\&<
        \frac{1}{M} L_{\textrm{close-$1$}}(\{(y_i, f_\theta(x_i))\}_{i=1}^{n})
        \\&<
        L_{0-1}(\{(y_i, f_\theta(x_i))\}_{i=1}^{n}) + k
    \end{align*}
\end{restatable}
See Appendix~\ref{sec:proofs} for the short proof.
As $k$ is decreased, this bound becomes tighter, which results in the close-$1$ loss being classification calibrated under restriction to any set.

\paragraph{Classification Calibration}

We now present the proof of \Cref{thm: calib}, which states that the close-1 loss is classification calibrated when restricted to any set of functions.
This proof consists of two main observations:
\begin{itemize}
    \item The minimizer of the close-$1$ loss matches the optimal 0-1 loss on any training set.
    \item It follows that in infinite population limit, we obtain the optimal 0-1 loss in $\Omega$.
\end{itemize}

We provide the first step of the proof as a lemma.
Informally, this lemma states that the minimizer of the close-$1$ aggregate loss achieves the true 0-1 optimum in $\Omega$.
\begin{lemma}
    \label{lemma: optimum}
    Let
    \begin{align*}
        \hat\theta = \argmin_{\theta} L_{\textrm{close-1}}(\{(y_i, f_\theta(x_i))\}_{i=1}^{n}).
    \end{align*}
    Then,
    \begin{align*}
        \min_\theta L_{0-1}(\{(y_i, f_\theta(x_i))\}_{i=1}^{n}) = L_{0-1}(\hat\theta)
    \end{align*}
\end{lemma}
\begin{proof}
    The optimal parameter choice is $\theta^* = \argmin_\theta L_{0-1}(\{(y_i, f_\theta(x_i))\}_{i=1}^{n})$, and corresponding 0-1 loss, which is the number of points classified incorrectly, is $w = \min_\theta L_{0-1}(\{(y_i, f_\theta(x_i))\}_{i=1}^{n})$.
    Then, $c = n - w$ is the number of points classified correctly by the optimal parameter choice.
    If the point closest to the decision boundary is classified correctly, then $L_{\textrm{close-1}}(\{(y_i, f_{\theta^*}(x_i))\}_{i=1}^{n})$
    is a sum consisting of $c-1$ 0's, a term less than $T$, and $w$ $M$'s.
    If the point closest to the decision boundary is classified incorrectly, then $L_{\textrm{close-1}}(\{(y_i, f_{\theta^*}(x_i))\}_{i=1}^{n})$
    is a sum consisting of $c$ 0's, a term between than $T$ and $M$, and $w - 1$ $M$'s.
    In both cases, $L_{\textrm{close-1}}(\{(y_i, f_{\theta^*}(x_i))\}_{i=1}^{n}) < wM$.

    Suppose that there is a $\hat\theta$ resulting in a smaller
     $L_{\textrm{close-1}}(\{(y_i, f_{\hat\theta}(x_i))\}_{i=1}^{n})$.
     Note that this is composed of a sum of $n$ non-negative elements, of which at most one element is not equal to either 0 or $M$.
     Then, for
     $L_{\textrm{close-1}}(\{(y_i, f_{\hat\theta}(x_i))\}_{i=1}^{n})$
     to be strictly less than $wM$, at least $c$ terms must be zero, so $\hat\theta$ would also get a 0-1 error of $w$.
\end{proof}

With this intermediate claim, we can now derive the proof of \Cref{thm: calib}.
\begin{proof}
    Using \Cref{lemma: optimum}, we have that for any training set, the minimizer of the close-$k$ loss also minimizes the 0-1 loss.
    In the infinite population limit,
    \begin{align*}
        \lim_{n\rightarrow\infty} \min_\theta L_{0-1}(\{(y_i, f_\theta(x_i))\}_{i=1}^{n}) = R^*.
    \end{align*}
    Using this observation with \Cref{lemma: optimum} gives us the desired claim.
\end{proof}

\paragraph{Generalization Guarantees}
Finally, we show that the close-$k$ aggregate loss is able to generalize well.
In particular, we find that the true accuracy of a learned model will be close to the training loss with high probability.
Additionally, as the number of training points $n$ increases, the generalization gap decays to 0 for $k=1$, where the loss is classification calibrated under restriction.
\begin{restatable}{theorem}{generalization}
    For any $\theta$, with probability at least $1 - \delta$,
    \begin{gather*}
        \tiny
        L_{0-1}
        \leq
        L_{\textrm{close-$k$}}(\{(y_i, f_{\theta}(x_i))\}_{i=1}^{n}) + k - 1 +\\ 
        \tilde O\Big(\sqrt{\left(L_{\textrm{close-$k$}}(\{(y_i, f_{\theta}(x_i))\}_{i=1}^{n}) + k - 1\right)\frac{\textrm{VC} - \log(\delta)}{n}}\\ + \frac{\textrm{VC} - \log(\delta)}{n}\Big).
    \end{gather*}
    where $VC$ is the VC dimension of the class of functions $\Omega$. 
\end{restatable}

The proof of this result relies on the bound from \Cref{lemma: k} along using results from learning theory to bound the gap between the training 0-1 loss and the test 0-1 loss with high probability \cite{boucheron2005edp}.
The full proof is provided in the appendix.

\section{Related Work}

Given the importance of the 0-1 loss in supervised classification tasks, there have been many approaches to improving the test accuracy in this setting.
Most closely related to our work are studies on other aggregate losses, such as the maximal and average top-$k$ losses \cite{shalevshwartz2016icml,fan2017nips}.
Other lines of work have focused on reweighing individual examples based on their difficulty \cite{lin2017iccv,ren2018icml}.
Similarly, robust estimation techniques identify outliers and downweigh their effect \cite{yang2010nips,huber1964aos}.
While these techniques can handle outliers well, they are not capable of handling other common classification settings that we explore.
Other studies on losses focus on particular models, such as SVMs \cite{huang2014jmlr,wu2007jasa,shen2003jasa} and boosting \cite{freund1997jcss}, which use specialized training schemes.
Finally, several papers study the case where the training set has noisy class labels, such as when training examples are either positive or unlabeled \cite{elkan2008kdd} or when training labels include random errors \cite{natarajan2013nips}.

Other lines of work study alternative loss functions in settings where the 0-1 loss is not an appropriate metric of performance.
For example, in information retrieval, the area under the ROC curve is commonly used \cite{kotlowski2011icml,eban2017aistats,natole2018icml}.
In other cases, false positives and false negatives have different costs \cite{domingos1999kdd}.

\section{Conclusion}

In this paper, we study several existing aggregate losses.
We find that even in simple datasets, existing aggregate losses result in suboptimal behavior, even on the training set and in the infinite population setting.
To explain this behavior, we propose classification calibration under restriction as a more precise definition.
As a result, we propose the close-$k$ aggregate loss, which allows a tradeoff between classification calibration under restriction and convexity.
We provide generalization bounds for the close-$k$ loss, and show that the close-$k$ loss does not result in unexpected overfitting.
We experimentally verify that the close-$k$ loss significantly improves performance on the PMLB and OpenML benchmark suites, and explain our improvements using variants of a real dataset.

An interesting problem that remains is generalizing the concept of classification calibration under restriction to the multiclass setting.
While it is straightforward to reduce multiclass classification problems to multiple binary classification problems using learning reductions, it is unclear whether the desired calibration properties would still hold after applying learning reduction techniques.
It may be interesting to apply similar concepts of focusing on examples near the decision boundary to multiclass problems to improve performance.

\paragraph{Code Availability}
An implementation of our method and code to reproduce the figures and tables in this paper are available on Github (\href{https://github.com/bryan-he/closek/}{https://github.com/bryan-he/closek/}).

\paragraph{Acknowledgements}
We thank Allen Nie, Amirali Aghazadeh, Jaime Roquero Gimenez, and Mika Sarkin Jain for helpful suggestions and feedback during development of this method.
The authors are supported by a Chan-Zuckerberg Investigator grant, National Science Foundation grant CRII 1657155, and National Science Foundation grant DGE 114747.

\clearpage
\bibliographystyle{plainnat}
\bibliography{bibliography}

\begin{thebibliography}{24}
\providecommand{\natexlab}[1]{#1}
\providecommand{\url}[1]{\texttt{#1}}
\expandafter\ifx\csname urlstyle\endcsname\relax
  \providecommand{\doi}[1]{doi: #1}\else
  \providecommand{\doi}{doi: \begingroup \urlstyle{rm}\Url}\fi

\bibitem[Alcal{\'a}-Fdez et~al.(2011)Alcal{\'a}-Fdez, Fern{\'a}ndez, Luengo,
  Derrac, Garc{\'\i}a, S{\'a}nchez, and Herrera]{alcala2011jmlsc}
Jes{\'u}s Alcal{\'a}-Fdez, Alberto Fern{\'a}ndez, Juli{\'a}n Luengo,
  Joaqu{\'\i}n Derrac, Salvador Garc{\'\i}a, Luciano S{\'a}nchez, and Francisco
  Herrera.
\newblock Keel data-mining software tool: data set repository, integration of
  algorithms and experimental analysis framework.
\newblock \emph{Journal of Multiple-Valued Logic \& Soft Computing}, 17, 2011.

\bibitem[Asuncion and Newman(2007)]{asuncion2007uci}
Arthur Asuncion and David Newman.
\newblock {UCI} machine learning repository, 2007.

\bibitem[Bartlett et~al.(2006)Bartlett, Jordan, and
  McAuliffe]{bartlett2006jasa}
Peter~L Bartlett, Michael~I Jordan, and Jon~D McAuliffe.
\newblock Convexity, classification, and risk bounds.
\newblock \emph{Journal of the American Statistical Association}, 101\penalty0
  (473):\penalty0 138--156, 2006.

\bibitem[Bischl et~al.(2017)Bischl, Casalicchio, Feurer, Hutter, Lang,
  Mantovani, van Rijn, and Vanschoren]{bischl2017arxiv}
Bernd Bischl, Giuseppe Casalicchio, Matthias Feurer, Frank Hutter, Michel Lang,
  Rafael~G Mantovani, Jan~N van Rijn, and Joaquin Vanschoren.
\newblock Openml benchmarking suites and the openml100.
\newblock \emph{arXiv preprint arXiv:1708.03731}, 2017.

\bibitem[Boucheron et~al.(2005)Boucheron, Bousquet, and
  Lugosi]{boucheron2005edp}
St{\'e}phane Boucheron, Olivier Bousquet, and G{\'a}bor Lugosi.
\newblock Theory of classification: A survey of some recent advances.
\newblock \emph{ESAIM: probability and statistics}, 9:\penalty0 323--375, 2005.

\bibitem[Domingos(1999)]{domingos1999kdd}
Pedro Domingos.
\newblock Metacost: A general method for making classifiers cost-sensitive.
\newblock In \emph{Proceedings of the fifth ACM SIGKDD international conference
  on Knowledge discovery and data mining}, pages 155--164. ACM, 1999.

\bibitem[Eban et~al.(2016)Eban, Schain, Mackey, Gordon, Saurous, and
  Elidan]{eban2017aistats}
Elad~ET Eban, Mariano Schain, Alan Mackey, Ariel Gordon, Rif~A Saurous, and Gal
  Elidan.
\newblock Scalable learning of non-decomposable objectives.
\newblock \emph{arXiv preprint arXiv:1608.04802}, 2016.

\bibitem[Elkan and Noto(2008)]{elkan2008kdd}
Charles Elkan and Keith Noto.
\newblock Learning classifiers from only positive and unlabeled data.
\newblock In \emph{Proceedings of the 14th ACM SIGKDD international conference
  on Knowledge discovery and data mining}, pages 213--220. ACM, 2008.

\bibitem[Fan et~al.(2017)Fan, Lyu, Ying, and Hu]{fan2017nips}
Yanbo Fan, Siwei Lyu, Yiming Ying, and Baogang Hu.
\newblock Learning with average top-k loss.
\newblock In \emph{Advances in Neural Information Processing Systems}, pages
  497--505, 2017.

\bibitem[Freund and Schapire(1997)]{freund1997jcss}
Yoav Freund and Robert~E Schapire.
\newblock A decision-theoretic generalization of on-line learning and an
  application to boosting.
\newblock \emph{Journal of computer and system sciences}, 55\penalty0
  (1):\penalty0 119--139, 1997.

\bibitem[Huang et~al.(2014)Huang, Shi, and Suykens]{huang2014jmlr}
Xiaolin Huang, Lei Shi, and Johan~AK Suykens.
\newblock Ramp loss linear programming support vector machine.
\newblock \emph{The Journal of Machine Learning Research}, 15\penalty0
  (1):\penalty0 2185--2211, 2014.

\bibitem[Huber et~al.(1964)]{huber1964aos}
Peter~J Huber et~al.
\newblock Robust estimation of a location parameter.
\newblock \emph{The annals of mathematical statistics}, 35\penalty0
  (1):\penalty0 73--101, 1964.

\bibitem[Kotlowski et~al.(2011)Kotlowski, Dembczynski, and
  Huellermeier]{kotlowski2011icml}
Wojciech Kotlowski, Krzysztof~J Dembczynski, and Eyke Huellermeier.
\newblock Bipartite ranking through minimization of univariate loss.
\newblock In \emph{Proceedings of the 28th International Conference on Machine
  Learning (ICML-11)}, pages 1113--1120. Citeseer, 2011.

\bibitem[Lin et~al.(2017)Lin, Goyal, Girshick, He, and Dollar]{lin2017iccv}
Tsung-Yi Lin, Priya Goyal, Ross Girshick, Kaiming He, and Piotr Dollar.
\newblock Focal loss for dense object detection.
\newblock In \emph{2017 IEEE International Conference on Computer Vision
  (ICCV)}, pages 2999--3007. IEEE, 2017.

\bibitem[Lin(2004)]{lin2004spl}
Yi~Lin.
\newblock A note on margin-based loss functions in classification.
\newblock \emph{Statistics \& probability letters}, 68\penalty0 (1):\penalty0
  73--82, 2004.

\bibitem[Natarajan et~al.(2013)Natarajan, Dhillon, Ravikumar, and
  Tewari]{natarajan2013nips}
Nagarajan Natarajan, Inderjit~S Dhillon, Pradeep~K Ravikumar, and Ambuj Tewari.
\newblock Learning with noisy labels.
\newblock In \emph{Advances in neural information processing systems}, pages
  1196--1204, 2013.

\bibitem[Natole et~al.(2018)Natole, Ying, and Lyu]{natole2018icml}
Michael Natole, Jr., Yiming Ying, and Siwei Lyu.
\newblock Stochastic proximal algorithms for {AUC} maximization.
\newblock In \emph{Proceedings of the 35th International Conference on Machine
  Learning}, pages 3710--3719, 2018.

\bibitem[Olson et~al.(2017)Olson, La~Cava, Orzechowski, Urbanowicz, and
  Moore]{olson2017biomed}
Randal~S Olson, William La~Cava, Patryk Orzechowski, Ryan~J Urbanowicz, and
  Jason~H Moore.
\newblock Pmlb: a large benchmark suite for machine learning evaluation and
  comparison.
\newblock \emph{BioData mining}, 10\penalty0 (1):\penalty0 36, 2017.

\bibitem[Ren et~al.(2018)Ren, Zeng, Yang, and Urtasun]{ren2018icml}
Mengye Ren, Wenyuan Zeng, Bin Yang, and Raquel Urtasun.
\newblock Learning to reweight examples for robust deep learning.
\newblock \emph{arXiv preprint arXiv:1803.09050}, 2018.

\bibitem[Shalev-Shwartz and Wexler(2016)]{shalevshwartz2016icml}
Shai Shalev-Shwartz and Yonatan Wexler.
\newblock Minimizing the maximal loss: How and why.
\newblock In \emph{ICML}, pages 793--801, 2016.

\bibitem[Shen et~al.(2003)Shen, Tseng, Zhang, and Wong]{shen2003jasa}
Xiaotong Shen, George~C Tseng, Xuegong Zhang, and Wing~Hung Wong.
\newblock On $\psi$-learning.
\newblock \emph{Journal of the American Statistical Association}, 98\penalty0
  (463):\penalty0 724--734, 2003.

\bibitem[Vapnik(1992)]{vapnik1992nips}
Vladimir Vapnik.
\newblock Principles of risk minimization for learning theory.
\newblock In \emph{Advances in neural information processing systems}, pages
  831--838, 1992.

\bibitem[Wu and Liu(2007)]{wu2007jasa}
Yichao Wu and Yufeng Liu.
\newblock Robust truncated hinge loss support vector machines.
\newblock \emph{Journal of the American Statistical Association}, 102\penalty0
  (479):\penalty0 974--983, 2007.

\bibitem[Yang et~al.(2010)Yang, Xu, White, Schuurmans, and Yu]{yang2010nips}
Min Yang, Linli Xu, Martha White, Dale Schuurmans, and Yao-liang Yu.
\newblock Relaxed clipping: A global training method for robust regression and
  classification.
\newblock In \emph{Advances in neural information processing systems}, pages
  2532--2540, 2010.

\end{thebibliography}

\clearpage
\appendix
\onecolumn
\section{Additional Experimental Results}

In this section, we provide additional results from our experiments on the PMLB and OpenML benchmarks.

\begin{table*}[htbp]
    \small
    \caption{Fraction of PMLB datasets with improvement of at least 2 percentage points.}
    \label{table: pmlb2}
    \begin{center}
\begin{tabular}{llcccccccccc}
\toprule
 &  & \multicolumn{5}{c}{Logistic} & \multicolumn{5}{c}{Hinge} \\
\cmidrule(lr){3-7}\cmidrule(lr){8-12} 
& & close & close decay & atk & average & top & close & close decay & atk & average & top \\
\midrule
\multirow{5}{*}{Linear}
& close       &      & 0.04 & 0.01 & 0.01 & 0.01 &      & 0.09 & 0.01 & 0.02 & 0.02 \\
& close decay & 0.02 &      & \textit{0.02} & \textit{0.02} & \textit{0.01} & 0.03 &      & \textit{0.01} & \textit{0.02} & \textit{0.02} \\
& atk         & 0.23 & \textbf{0.25} &      & 0.04 & 0.06 & 0.24 & \textbf{0.24} &      & 0.10 & 0.03 \\
& average     & 0.18 & \textbf{0.23} & 0.03 &      & 0.06 & 0.17 & \textbf{0.18} & 0.04 &      & 0.06 \\
& top         & 0.35 & \textbf{0.34} & 0.14 & 0.20 &      & 0.29 & \textbf{0.33} & 0.10 & 0.15 &      \\
\midrule
\multirow{5}{*}{NN}
& close       &      & 0.09 & 0.04 & 0.04 & 0.00 &      & 0.08 & 0.07 & 0.04 & 0.03 \\
& close decay & 0.06 &      & \textit{0.03} & \textit{0.02} & \textit{0.03} & 0.03 &      & \textit{0.02} & \textit{0.03} & \textit{0.02} \\
& atk         & 0.08 & \textbf{0.11} &      & 0.07 & 0.04 & 0.08 & \textbf{0.09} &      & 0.04 & 0.03 \\
& average     & 0.10 & \textbf{0.11} & 0.02 &      & 0.06 & 0.08 & \textbf{0.09} & 0.08 &      & 0.07 \\
& top         & 0.29 & \textbf{0.29} & 0.23 & 0.28 &      & 0.29 & \textbf{0.32} & 0.27 & 0.32 &      \\
\bottomrule
\end{tabular}
\end{center}

\end{table*}

\begin{table*}[htbp]
    \small
    \caption{Fraction of OpenML datasets with improvement of at least 2 percentage points.}
    \label{table: openml2}
    \begin{center}
\begin{tabular}{llcccccccccc}
\toprule
 &  & \multicolumn{5}{c}{Logistic} & \multicolumn{5}{c}{Hinge} \\
\cmidrule(lr){3-7}\cmidrule(lr){8-12} 
& & close & close decay & atk & average & top & close & close decay & atk & average & top \\
\midrule
\multirow{5}{*}{Linear}
& close       &      & 0.04 & 0.00 & 0.00 & 0.00 &      & 0.04 & 0.00 & 0.00 & 0.00 \\
& close decay & 0.00 &      & \textit{0.00} & \textit{0.00} & \textit{0.00} & 0.00 &      & \textit{0.00} & \textit{0.00} & \textit{0.00} \\
& atk         & 0.19 & \textbf{0.19} &      & 0.00 & 0.12 & 0.15 & \textbf{0.19} &      & 0.00 & 0.08 \\
& average     & 0.19 & \textbf{0.19} & 0.00 &      & 0.15 & 0.15 & \textbf{0.19} & 0.00 &      & 0.04 \\
& top         & 0.19 & \textbf{0.23} & 0.00 & 0.00 &      & 0.23 & \textbf{0.27} & 0.04 & 0.04 &      \\
\midrule
\multirow{5}{*}{NN}
& close       &      & 0.02 & 0.00 & 0.00 & 0.02 &      & 0.11 & 0.05 & 0.07 & 0.02 \\
& close decay & 0.00 &      & \textit{0.00} & \textit{0.00} & \textit{0.00} & 0.02 &      & \textit{0.00} & \textit{0.02} & \textit{0.00} \\
& atk         & 0.05 & \textbf{0.05} &      & 0.02 & 0.02 & 0.05 & \textbf{0.05} &      & 0.02 & 0.00 \\
& average     & 0.02 & \textbf{0.02} & 0.00 &      & 0.00 & 0.07 & \textbf{0.07} & 0.02 &      & 0.02 \\
& top         & 0.34 & \textbf{0.34} & 0.32 & 0.34 &      & 0.34 & \textbf{0.34} & 0.30 & 0.34 &      \\
\bottomrule
\end{tabular}
\end{center}

\end{table*}

\section{Proofs}\label{sec:proofs}

\paragraph{Proof of Lemma~\ref{lem:bound}}.
\begin{proof}
    Notice that at most $k$ terms in the two summations differ.
    The terms that differs must be between 0 and 1, so the total difference must be between $-k$ and $k$.
\end{proof}

\lemmak*
\begin{proof}
    The optimal parameter choice is $\theta^* = \argmin_\theta L_{0-1}(\{(y_i, f_\theta(x_i))\}_{i=1}^{n})$, and
    the number of points classified incorrectly is $w = \min_\theta L_{0-1}(\{(y_i, f_\theta(x_i))\}_{i=1}^{n})$.
    Then, $c = n - w$ is the number of points classified correctly by the optimal parameter choice.

    $L_{\textrm{close-$k$}}(\{(y_i, f_{\theta^*}(x_i))\}_{i=1}^{n})$ is a sum consisting of at least $c - k$ 0's, $k$ terms between $0$ and $M$, and at most $n - c - k$ $M$'s.
    Thus, $L_{\textrm{close-$k$}}(\{(y_i, f_{\theta^*}(x_i))\}_{i=1}^{n}) \leq (n - c)M = wM$.

    Suppose that there is a $\hat\theta$ resulting in a smaller
     $L_{\textrm{close-1}}(\{(y_i, f_{\hat\theta}(x_i))\}_{i=1}^{n})$.
     Note that this is composed of a sum of $n$ non-negative elements that are no more than $M$, of which at most $k$ elements is not equal to either 0 or $M$.
     Then, for
     $L_{\textrm{close-1}}(\{(y_i, f_{\hat\theta}(x_i))\}_{i=1}^{n})$
     to be strictly less than $wM$, at least $c-k+1$ terms must be zero, so $\hat\theta$ would also get a 0-1 error of $w$.
\end{proof}

\generalization*
\begin{proof}
    From Lemma 3, we have that
    $L_{\textrm{close-$k$}}(\{(y_i, f_\theta(x_i))\}_{i=1}^{n}) + k - 1 \leq L_{0-1}(\{(y_i, f_\theta(x_i))\}_{i=1}^{n})$.
    In addition, we have that
    \begin{gather*}
        L_{0-1} \leq
        L_{0-1}(\{(y_i, f_{\theta^*}(x_i))\}_{i=1}^{n}) + k - 1 + \\
        \tilde O\Big(\sqrt{\left(L_{0-1}(\{(y_i, f_{\theta^*}(x_i))\}_{i=1}^{n}) + k - 1\right)\frac{\textrm{VC} - \log(\delta)}{m}} \\+ \frac{\textrm{VC} - \log(\delta)}{m}\Big).
    \end{gather*}
    from Corollary 5.2 of \citet{boucheron2005edp}.
    Combining these two inequalities results in the desired result.
\end{proof}

\end{document}